\documentclass[conference,a4paper]{IEEEtran2}

\usepackage{amsmath}
\usepackage{amsfonts}
\usepackage{amssymb}
\usepackage{amsbsy}
\usepackage{graphicx}
\usepackage{times}
\usepackage{default}
\usepackage{color}
\usepackage[noadjust]{cite}
\usepackage{enumerate}
\usepackage{tikz}
\usepackage{subcaption}
\usepackage{amsthm}
\usepackage{dsfont}
\usepackage{url}
\usetikzlibrary{arrows}

\newcommand{\la}{\lambda}
\newcommand{\R}{\mathbb{R}}
\newcommand{\Rd}{\mathbb{R}^d}

\newcommand{\C}{\mathbb{C}}

\def\ba#1\ea{\begin{align*}#1\end{align*}}	
\def\ban#1\ean{\begin{align}#1\end{align}}	

\newtheorem{theorem}{Theorem}

\newtheorem{lemma}{Lemma}
\newtheorem{definition}{Definition}
\newtheorem{proposition}{Proposition}
\newtheorem{corollary}{Corollary}
\newtheorem{remark}{Remark}
\IEEEoverridecommandlockouts 

\makeatletter
\def\@IEEEsectpunct{.\ \,}
\makeatother

\title{Deep Convolutional Neural Networks\\ on Cartoon Functions}

\author{\IEEEauthorblockN{Philipp Grohs\IEEEauthorrefmark{1},
Thomas Wiatowski\IEEEauthorrefmark{2}, and Helmut B\"olcskei\IEEEauthorrefmark{2}}
\IEEEauthorblockA{\IEEEauthorrefmark{1}Dept. Math., ETH Zurich, Switzerland, and Dept. Math., University of Vienna, Austria\\
\IEEEauthorrefmark{2}Dept. IT \& EE, ETH Zurich, Switzerland,\\
\IEEEauthorrefmark{1}philipp.grohs@sam.math.ethz.ch, \IEEEauthorrefmark{2}\{withomas, boelcskei\}@nari.ee.ethz.ch}}

\begin{document}
\maketitle

\begin{abstract} Wiatowski and B\"olcskei, 2015, proved that deformation stability and vertical translation invariance of deep convolutional neural network-based feature extractors are guaranteed by the network structure per se rather than the specific convolution kernels and non-linearities. 
While the translation invariance result applies to square-integrable functions, the deformation stability bound holds for band-limited functions only. Many signals of practical relevance (such as natural images) exhibit, however, sharp and curved discontinuities and are, hence, not band-limited. The main contribution of this paper is a deformation stability result that takes these structural properties into account. Specifically, we establish deformation stability bounds for the class of cartoon functions introduced by Donoho, 2001. 

\end{abstract}

\section{Introduction}

Feature extractors based on so-called deep convolutional neural networks have been applied with tremendous success in a wide range of practical signal classification tasks \cite{Nature}. These networks are composed of multiple layers, each of which computes convolutional transforms, followed by the application of non-linearities and pooling operations. 

The mathematical ana\-lysis of feature extractors
gene\-rated by deep convolutional neural networks was initiated in a se\-minal paper by Mallat \cite{MallatS}. Specifically, Mallat analyzes so-called scattering networks, where signals are propagated through layers that compute semi-discrete wavelet transforms (i.e., convolutional transforms with pre-specified filters obtained from a mother wavelet through scaling operations), followed by modulus non-linearities. It was shown in \cite{MallatS} that the resulting wavelet-modulus feature extractor is horizontally translation-invariant \cite{Wiatowski_journal} 
and deformation-stable, with the stability result applying to a function space that depends on the underlying mother wavelet.

Recently, Wiatowski and B\"olcskei \cite{Wiatowski_journal} extended Mallat's  theory to incorporate convolutional transforms with filters that are (i) pre-specified and potentially structured such as Weyl-Heisenberg (Gabor) functions \cite{Groechenig}, wavelets \cite{Daubechies}, cur\-ve\-lets \cite{CandesDonoho2},  shearlets \cite{Shearlets}, and ridge\-lets \cite{Ridgelet}, (ii) pre-specified and unstructured such as random filters \cite{Jarrett}, and (iii) learned in a supervised \cite{Huang} or unsupervised \cite{hierachies} fashion. Furthermore, the networks in \cite{Wiatowski_journal} may employ gene\-ral Lipschitz-continuous non-linearities (e.g., rectified linear units, shifted logistic sigmoids, hyperbolic tangents, and the modulus function) and pooling through sub-sampling. The essence of the results in \cite{Wiatowski_journal} is that vertical translation invariance and deformation stability 
are induced by the network structure per se rather than the specific choice of filters and non-linearities. While the vertical translation invariance result in \cite{Wiatowski_journal} is general in the sense of applying to the  function space $L^2(\Rd)$, the deformation stability result in \cite{Wiatowski_journal} pertains to square-integrable band-limited functions. Moreover,  the corresponding deformation stability bound depends linearly on the bandwidth. 

Many signals of practical relevance (such as natural ima\-ges) can be  modeled as square-integrable functions that are, however, not band-limited or have large bandwidth. Large bandwidths render the deformation stability bound in \cite{Wiatowski_journal} void as a consequence of its linear dependence on bandwidth. 

\paragraph*{Contributions} The question considered in this paper is whether taking structural properties of natural images into account can lead to stronger deformation stability bounds. We show that the answer is in the affirmative by analyzing the class of cartoon functions introduced in \cite{Cartoon}. Cartoon functions satisfy mild decay properties and are piecewise continuously differentiable apart from curved discontinuities along $C^2$-hypersurfaces. Moreover, they provide a good model for natural images such as those in the MNIST \cite{MNIST}, Caltech-256 \cite{Caltech256}, and CIFAR-100 \cite{CIFAR2} datasets as well as for images of geometric objects of different shapes, sizes, and colors \cite{BABYAI}. The proof of our main result is based on the decoupling technique introduced in   \cite{Wiatowski_journal}. The essence of decoupling is that contractivity of the feature extractor combined with deformation stability of the signal class under consideration---under smoothness conditions on the deformation---establishes deformation stability for the feature extractor. Our main technical contribution here is to prove deformation stability for the class of cartoon functions. Moreover, we show that the decay rate of the resulting deformation stability bound is best possible. The results we obtain further underpin the observation made in \cite{Wiatowski_journal} of deformation stability and vertical translation invariance being induced by the network structure per se.
\begin{figure*}
\centering
\begin{tikzpicture}[scale=2,level distance=10mm,>=angle 60]

  \tikzstyle{every node}=[rectangle, inner sep=1pt]
  \tikzstyle{level 1}=[sibling distance=30mm]
  \tikzstyle{level 2}=[sibling distance=10mm]
  \tikzstyle{level 3}=[sibling distance=4mm]
  \node {$U[e]f=f$}
	child[grow=90, level distance=.45cm] {[fill=gray!50!black] circle (0.5pt)
		child[grow=130,level distance=0.5cm] 
        		child[grow=90,level distance=0.5cm] 
        		child[grow=50,level distance=0.5cm]
		child[level distance=.25cm,grow=215, densely dashed, ->] {}  
	}
        child[grow=150] {node {$U\big[\lambda_1^{(j)}\big]f$}
	child[level distance=.75cm,grow=215, densely dashed, ->] {node {$\big(U\big[\lambda_1^{(j)}\big]f\big)\ast\chi_{1}$}
	}
	child[grow=83, level distance=0.5cm] 
	child[grow=97, level distance=0.5cm] 
        child[grow=110] {node {$U\big[\big(\lambda_1^{(j)},\lambda_2^{(l)}\big)\big]f$}
	child[level distance=1cm,grow=215, densely dashed, ->] {node {$\big(U\big[\big(\lambda_1^{(j)},\lambda_2^{(l)}\big)\big]f\big)\ast\chi_{2}$}%
	}
        child[grow=130] {node {$U\big[\big(\lambda_1^{(j)},\lambda_2^{(l)},\lambda_3^{(m)}\big)\big]f$}
	child[level distance=0.75cm,grow=215, densely dashed, ->] {node {}}
	}
        child[grow=90,level distance=0.5cm]
 	child[grow=50,level distance=0.5cm]
       }
       child[grow=63, level distance=1.05cm] {[fill=gray!50!black] circle (0.5pt)
	child[grow=130,level distance=0.5cm] 
       child[grow=90,level distance=0.5cm] 
       child[grow=50,level distance=0.5cm]
       child[level distance=.25cm,grow=325, densely dashed, ->] {}    
       }
       }
       child[grow=30] {node {$U\big[\lambda_1^{(p)}\big]f$}
       child[level distance=0.75cm, grow=325, densely dashed, ->] {node {$\big(U\big[\lambda_1^{(p)}\big]f\big)\ast\chi_{1}$}
	}
	child[grow=83, level distance=0.5cm] 
	child[grow=97, level distance=0.5cm]
        child[grow=117, level distance=1.05cm] {[fill=gray!50!black] circle (0.5pt)
        child[grow=130,level distance=0.5cm] 
        child[grow=90,level distance=0.5cm] 
        child[grow=50,level distance=0.5cm] 
        child[level distance=.25cm,grow=215, densely dashed, ->] {}  
	 }
        child[grow=70] {node {$U\big[\big(\lambda_1^{(p)},\lambda_2^{(r)}\big)\big]f$}
	 child[level distance=1cm,grow=325, densely dashed, ->] {node {$\big(U\big[\big(\lambda_1^{(p)},\lambda_2^{(r)}\big)\big]f\big)\ast\chi_{2}$}}
	child[grow=130,level distance=0.5cm] 
         child[grow=90,level distance=0.5cm] 
             child[grow=50] {node {$U\big[\big(\lambda_1^{(p)},\lambda_2^{(r)},\lambda_3^{(s)}\big)\big]f$}
             child[level distance=0.75cm,grow=325, densely dashed, ->] {node {}}}
	}
     }
	child[level distance=0.75cm, grow=215, densely dashed, ->] {node {$f\ast \chi_0$}};
\end{tikzpicture}
\caption{Network architecture underlying the feature extractor  \eqref{ST}. The index $\lambda_{n}^{(k)}$ corresponds to the $k$-th atom $g_{\lambda_{n}^{(k)}}$ of the collection $\Psi_n$ associated with the $n$-th network layer. The function $\chi_{n}$ is the output-generating atom of the $n$-th layer.} 
\label{fig:gsn}
\end{figure*}
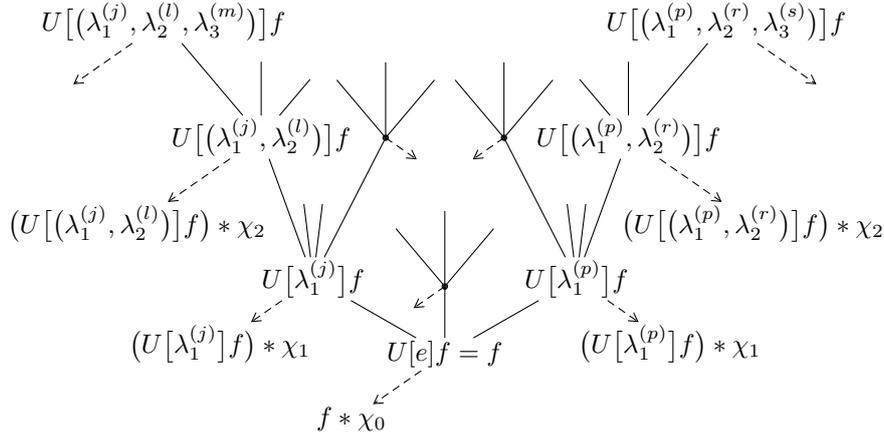
\paragraph*{Notation}
We refer the reader to \cite[Sec. 1]{Wiatowski_journal} for the general notation employed in this paper. In addition, we will need the following notation.  For $x\in \mathbb{R}^d$, we set $\langle x\rangle := (1+|x|^2)^{1/2}$. The Minkowski sum of sets $A,B\subseteq \R^d$ is $(A+B):=\{ a+b \, | \, a\in A, \ b \in B \}$. The indicator function
 of a set $B\subseteq \Rd$ is defined as $\mathds{1}_B(x):=1$, for $x\in B$, and $\mathds{1}_B(x):=0$, for $x\in \Rd\backslash B$. For a measurable set $B\subseteq \R^d$, we let $\mbox{vol}^{d}(B):=\int_{\R^d}\mathds{1}_B(x)\mathrm dx=\int_{B}1\mathrm dx$. 
\section{Deep convolutional neural network-based feature extractors}\label{architecture}
We set the stage by briefly reviewing the deep convolutional feature extraction network presented in  \cite{Wiatowski_journal}, the basis of which is a sequence of triplets
$
\Omega:=\big( (\Psi_n,M_n,R_n)\big)_{n \in \mathbb{N}}
$
referred to as module-sequence. The triplet $(\Psi_n,M_n,R_n)$---associated with the $n$-th network layer---consists of (i) a collection $\Psi_n:=\{ g_{\la_n}\}_{\lambda_n  \in \Lambda_n}$ of so-called atoms $g_{\lambda_n}\in L^1(\Rd) \cap L^2(\Rd)$, indexed by a countable set $\Lambda_n$ and satisfying the Bessel condition $\sum_{\lambda_n \in \Lambda_n}\| f\ast g_{\lambda_n}\|^2\leq B_n\| f\|_2^2$, for all $f\in L^2(\Rd)$, for some $B_n>0$, 
(ii) an operator $M_n:L^2(\Rd)\to L^2(\Rd)$ satisfying the Lipschitz property $\| M_nf-M_nh\|_2 \leq L_n \| f-h \|_2$, for all $f,h\in L^2(\Rd)$, and $M_nf=0$ for $f=0$, and (iii) a sub-sampling factor $R_n\geq 1$. Associated with $(\Psi_n,M_n,R_n)$, we define the operator  
 \begin{equation}\label{eq:1}
U_n[\lambda_n]f:=R_n^{d/2}\big(M_n(f\ast g_{\la_{n}})\big)(R_n\cdot),
\end{equation}   
and extend it to paths on index sets $q=(\lambda_1,\lambda_{2},\dots, \lambda_n) \in \Lambda_1\times \Lambda_{2}\times \dots \times \Lambda_n:=\Lambda_1^n$, $n \in \mathbb{N}$, according to
\begin{equation*}\label{aaaa}
\begin{split}
U[q]f=&\,U[(\lambda_1,\lambda_{2},\dots,\lambda_n)]f\\:=&\, U_n[\lambda_n] \cdots U_{2}[\lambda_{2}]U_{1}[\lambda_{1}]f,
\end{split}
\end{equation*}
where for the empty path $e:=\emptyset$ we set $\Lambda_1^0:=\{ e \}$ and $U[e]f:=f$, for  $f\in L^2(\Rd)$. 

\begin{remark}
The Bessel condition on the atoms  $g_{\lambda_n}$ is equi\-valent to $\sum_{\lambda_n \in \Lambda_n} |\widehat{g_{\lambda_n}}(\omega)|^2\leq B_n,$ for a.e. $\omega \in \Rd$ (see \cite[Prop. 2]{Wiatowski_journal}), and is hence easily satisfied even by learned filters \cite[Remark 2]{Wiatowski_journal}. An overview of collections $\Psi_n=\{ g_{\lambda_n}\}_{\lambda_n\in \Lambda_n}$ of structured atoms $ g_{\lambda_n}$ (such as, e.g., Weyl-Heisenberg  (Gabor) functions, wavelets, curvelets, shearlets, and ridgelets) and non-linearities $M_n$ widely used in the deep learning literature (e.g., hyperbolic tangent, shifted logistic sigmoid, rectified linear unit, and modulus function) is provided in \cite[App. B-D]{Wiatowski_journal}. 
\end{remark}

For every $n\in \mathbb{N}$, we designate one of the atoms $\Psi_n=\{ g_{\lambda_n}\}_{\lambda_n \in \Lambda_n}$ as the output-generating atom $\chi_{n-1}:=g_{\lambda^\ast_n}$, $\lambda^\ast_n \in \Lambda_n$, of the $(n-1)$-th layer. The atoms $\{ g_{\lambda_n}\}_{\lambda_n \in \Lambda_n\backslash\{ \lambda^\ast_n\}}\cup \{ \chi_{n-1}\}$ are thus used across two consecutive layers in the sense of $\chi_{n-1}=g_{\lambda^\ast_n}$ generating the output in the $(n-1)$-th layer, and the remaining atoms $\{ g_{\lambda_n}\}_{\lambda_n \in \Lambda_n\backslash\{ \lambda^\ast_n\}}$  propagating signals to the $n$-th layer according to \eqref{eq:1}, see Fig. \ref{fig:gsn}. From now on, with slight abuse of notation, we write $\Lambda_n$ for $\Lambda_n\backslash\{ \lambda^\ast_n\}$ as well. 

The extracted features $\Phi_\Omega(f)$ of a signal $f\in L^2(\Rd)$ are defined as \cite[Def. 3]{Wiatowski_journal}
\begin{equation}\label{ST}
\Phi_\Omega (f):=\bigcup_{n=0}^\infty\{ (U[q]f) \ast \chi_{n} \}_{q \in \Lambda_1^n},
\end{equation}
where $(U[q]f) \ast \chi_{n}$, $q\in \Lambda_1^{n}$, is a feature generated in the $n$-th layer of the network, see Fig. \ref{fig:gsn}. 
It is shown in \cite[Thm. 2]{Wiatowski_journal} that for all $f \in L^2(\Rd)$ the feature extractor $\Phi_\Omega$ is vertically translation-invariant in the sense of the layer depth $n$ determining the extent to which the features $(U[q]f) \ast \chi_{n}$, $q\in \Lambda_1^{n}$,  are translation-invariant. 
Furthermore, under the condition
\begin{equation}\label{weak_admiss2}
 \max_{n\in\mathbb{N}}\max\{B_n,B_nL_n^2 \}\leq 1, 
 \end{equation}
 referred to as \textit{weak admissibility condition} in \cite[Def. 4]{Wiatowski_journal} and satisfied by a wide variety of module sequences $\Omega$ (see \cite[Sec. 3]{Wiatowski_journal}), the following  result is established in \cite[Thm. 1]{Wiatowski_journal}: The feature extractor $\Phi_\Omega$ is deformation-stable on the space of $R$-band-limited functions $L^2_R(\Rd)$ w.r.t. deformations $(F_{\tau} f)(x):=f(x-\tau(x))$, i.e., there exists a universal constant $C>0$ (that does not depend on $\Omega$) such that for all $f \in L^2_R(\Rd)$ and all (possibly non-linear) $\tau \in C^1(\Rd,\Rd)$ with $\| D \tau \|_\infty\leq\frac{1}{2d}$, it holds that
\begin{equation}\label{eq:oldstab}
||| \Phi_\Omega(F_{\tau} f)-\Phi_\Omega(f) |||\leq C R\| \tau \|_\infty \| f \|_2.
\end{equation}
Here, the feature space norm is defined as $||| \Phi_\Omega(f)|||^2:=\sum_{n=0}^{\infty}\sum_{q\in \Lambda_{1}^n}\| (U[q]f)\ast \chi_n\|_2^2$.\\

For practical classification tasks, we can think of the deformation $F_\tau$ as follows. Let $f$ be a representative of a certain signal class, e.g., $f$ is an image of the handwritten digit ``$8$'' (see Fig. \ref{fig:data}, right). Then, $\{ F_\tau f \ | \ \| D\tau \|_\infty<\frac{1}{2d}\}$ is a collection of images of the handwritten digit ``$8$'', where each $F_\tau f$ may be generated, e.g., based on a different handwriting style. The bound $\| D\tau \|_\infty<\frac{1}{2d}$ on the Jacobian matrix of $\tau$ imposes a quantitative limit on the amount of deformation tolerated, rendering the bound \eqref{eq:oldstab} to implicitly depend on $D\tau$. The deformation stability bound \eqref{eq:oldstab} now guarantees that the features corresponding to the images in the set $\{ F_\tau f \ | \ \| D\tau \|_\infty<\frac{1}{2d}\}$ do not differ too much.
\begin{figure}
\centering
	\includegraphics[width = .2\textwidth]{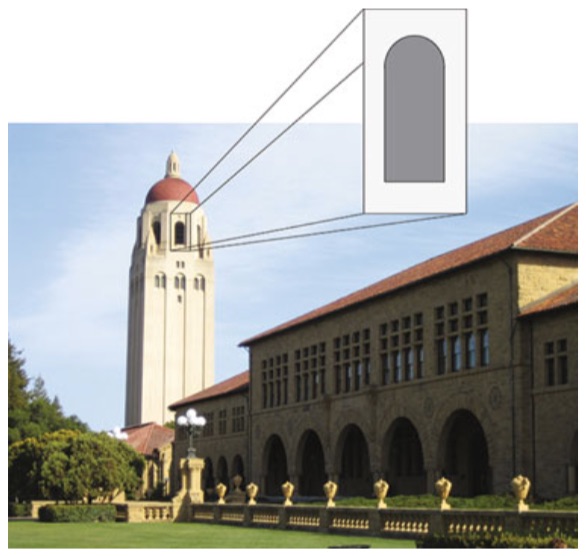}
	\includegraphics[width = .2\textwidth]{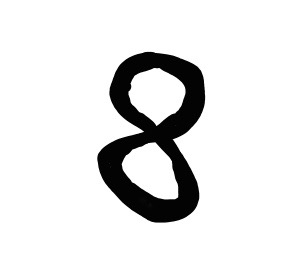}
	\caption{Left: A natural image (image credit: \cite{ShearletsIntro}) is typically governed by  areas of little variation, with the individual areas separated by edges that can be modeled as curved singularities. Right: An image of a handwritten digit. }
	\label{fig:data}
\end{figure}
\section{Cartoon functions}\label{main}
The bound in \eqref{eq:oldstab} applies to the space of square-integrable $R$-band-limited functions. Many signals of practical significance (e.g., natural images) are, however, not band-limited (due to the presence of sharp and possibly curved edges, see Fig. \ref{fig:data}) or exhibit large bandwidths. In the latter case, the deformation stability bound \eqref{eq:oldstab} becomes void as it depends linearly on $R$. 

The goal of this paper is to take structural properties of natural images into account by considering  the class of cartoon functions introduced in \cite{Cartoon}. These functions satisfy mild decay properties and are piecewise continuously differentiable apart from curved discontinuities along $C^2$-hypersurfaces. Cartoon functions provide a good model for natural images (see Fig. \ref{fig:data}, left) such as those in the Caltech-256 \cite{Caltech256} and CIFAR-100 \cite{CIFAR2} data sets, for images of handwritten digits \cite{MNIST} (see Fig. \ref{fig:data}, right), and for images of geometric objects of different shapes, sizes, and colors \cite{BABYAI}.

We will work with the following---relative to the definition in \cite{Cartoon}---slightly modified version of cartoon functions. 

\begin{definition}\label{def:cartoon}
The function $f:\R^d \to \mathbb{C}$ is referred to as a cartoon function if it can be written as $f = f_1 + \mathds{1}_Bf_2$, where $B\subseteq \R^d$ is a compact domain whose boundary $\partial B$ is a compact topologically embedded $C^2$-hypersurface of $\Rd$ without boundary\footnote{We refer the reader to \cite[Chapter 0]{RiemannianBook} for a review on differentiable manifolds.}, and $f_i\in L^2(\Rd)\cap C^1(\R^d,\mathbb{C})$, $i=1,2$, satisfy the decay condition 
\begin{align}\label{eq:decay}
|\nabla f_i(x)|\le C \langle x\rangle^{-d},\quad i=1,2,
\end{align}
for some $C>0$ (not depending on $f_1$,$f_2$). Furthermore, we denote by 
\begin{align*}
&\mathcal{C}^{K}_{\mathrm{CART}}:=\{f_1 + \mathds{1}_Bf_2 \ | \  f_i\in L^2(\Rd)\cap C^1(\R^d,\mathbb{C}), \ i=1,2,\\ 
 &|\nabla f_i(x)|\le K \langle x\rangle^{-d}, \ \emph{vol}^{d-1}(\partial B)\leq K, \ \|f_2\|_\infty \leq K\}
\end{align*} 
the class of cartoon functions of  ``size'' $K>0$.
\end{definition}

%
We chose the term ``size'' to indicate the length $\mbox{vol}^{d-1}(\partial B)$ of the hypersurface  $\partial B$. Furthermore,  $\mathcal{C}^{K}_{\mathrm{CART}}\subseteq L^2(\Rd)$, for all $K>0$; this  follows from the triangle inequality according to $\|f_1+\mathds{1}_Bf_2 \|_2\leq \|f_1 \|_2 + \| \mathds{1}_B f_2\|_2\leq \|f_1 \|_2+ \| f_2\|_2<\infty$, where in the last step we used $f_1,f_2\in L^2(\Rd)$. Finally, we note that our  main results---presented in the next section---can easily be generalized to finite linear combinations of cartoon functions, but this is not done here for simplicity of exposition. 

\section{Main results}
We start by reviewing the decoupling technique introduced in \cite{Wiatowski_journal} to prove deformation stability bounds for band-limited functions. The proof of the deformation stability bound \eqref{eq:oldstab} for band-limited functions in \cite{Wiatowski_journal} is based on two key ingredients. The first one is a contractivity property of $\Phi_\Omega$ (see \cite[Prop. 4]{Wiatowski_journal}), namely $||| \Phi_\Omega (f) - \Phi_\Omega(h) |||\leq \| f-h \|_2$, for all $f,h \in L^2(\Rd)$. Contractivity guarantees that pairwise distances of input signals do not increase through feature extraction. The second ingredient is an upper bound on the deformation error $\| f-F_\tau f \|_2$ (see \cite[Prop. 5]{Wiatowski_journal}), specific to the signal class considered in \cite{Wiatowski_journal}, namely band-limited functions. Re\-cognizing that the combination of these two ingredients yields a simple proof of deformation stability is interesting as it shows that whenever a signal class exhibits inherent stability w.r.t. deformations of the form $(F_\tau f)(x)=f(x-\tau(x))$, we automatically obtain deformation stability for the feature extractor $\Phi_\Omega$. The present paper employs this decoupling technique and establishes deformation stability for the class of cartoon functions by deriving an upper bound on the deformation error $\| f-F_\tau f \|_2$ for $f\in \mathcal{C}^{K}_{\mathrm{CART}}$. 


\begin{proposition}\label{prop:main}
For every $K>0$, there exists a constant $C_K>0$ such that for all $f\in \mathcal{C}^K_{\mathrm{CART}}$ and all (possibly non-linear) $\tau:\R^d \to \R^d$ with $\| \tau \|_\infty<\frac{1}{2}$, it holds that 
  \begin{equation}\label{eq:main}ƒ
  	\|f - F_\tau f \|_2\le C_K\|\tau\|_\infty^{1/2}.
  \end{equation}
\end{proposition}
\begin{proof}[Proof]
See Appendix \ref{app:prothm1}.
\end{proof}
The Lipschitz exponent $\alpha=\frac{1}{2}$ on the right-hand side (RHS) of \eqref{eq:main} determines the decay rate of the deformation error $\|f - F_\tau f \|_2$  as $\| \tau \|_\infty\to 0$. Clearly, larger $\alpha>0$ results in the deformation error decaying faster as the deformation becomes smaller. The following simple example shows that the Lipschitz exponent $\alpha=\frac{1}{2}$ in \eqref{eq:main} is best possible, i.e., it can not be larger. Consider $d=1$ and $\tau_s(x)= s$,  for a fixed $s$ satisfying $0<s<\frac{1}{2}$; the corresponding deformation $F_{\tau_s}$ amounts to a simple translation by $s$ with $\| \tau_s \|_\infty=s<\frac{1}{2}$. Let $f=\mathds{1}_{[-1,1]}$. Then,  $f\in \mathcal{C}^K_{\mathrm{CART}}$ for some $K>0$ and 
$\|f - F_{\tau_s}f\|_2= \sqrt{2s}=\sqrt{2}\| \tau \|^{1/2}_\infty.$
\begin{remark}\label{remark_int}
It is interesting to note that in order to obtain bounds of the form $\|f - F_\tau f \|_2\le C\|\tau\|_\infty^\alpha$, for $f\in \mathcal{C}\subseteq L^2(\R^d)$, for some $C>0$ (that does not depend on $f$, $\tau$) and some $\alpha> 0$, we need to impose \textit{non-trivial} constraints on the set  $\mathcal{C}\subseteq L^2(\R^d)$. Indeed, consider, again, $d=1$ and $\tau_s(x) = s$,  for small $s>0$. Let $f_s \in L^2(\Rd)$ be a function that has its energy $\|f_s\|_2=1$ concentrated in a small interval according to $\mbox{supp}(f_s)\subseteq [-s/2,s/2]$. Then, $f_s$ and $F_{\tau_s}f_s$ have disjoint support sets and hence 
$\|f_s - F_{\tau_s}f_s\|_2= \sqrt{2},$
which does not decay with $\|\tau\|_\infty^\alpha = s^\alpha$ for any $\alpha>0$. More generally, the amount of deformation induced by a given function $\tau$ depends strongly on the signal (class) it is applied to. Concretely, the deformation $F_\tau$ with  $\tau(x)=e^{-x^2}$, $x\in \mathbb{R}$, will lead to a small bump around the origin only when applied to a low-pass function, whereas the function $f_s$ above will experience a significant deformation.
\end{remark}
We are now ready to state our main result.
\begin{theorem}\label{mainmain}
 Let $\Omega=\big( (\Psi_n,M_n,R_n)\big)_{n \in \mathbb{N}}$ be a module-sequence satisfying the weak admissibility condition \eqref{weak_admiss2}. For every size $K>0$, the feature extractor $\Phi_\Omega$ is deformation-stable on the space of cartoon functions  $\mathcal{C}^K_{\mathrm{CART}}$ w.r.t. deformations $(F_{\tau} f)(x)=f(x-\tau(x))$, i.e., for every $K>0$, there exists a constant $C_K>0$ (that does not depend on $\Omega$) such that for all $f \in \mathcal{C}^K_{\mathrm{CART}}$, and all (possibly non-linear) $\tau \in C^1(\Rd,\Rd)$ with $\| \tau\|_\infty <\frac{1}{2}$ and $\| D \tau \|_\infty\leq\frac{1}{2d}$, it holds that
\begin{equation}\label{mainmainmain}
||| \Phi_\Omega(F_{\tau} f)-\Phi_\Omega(f) |||\leq C_K \| \tau \|_\infty^{1/2}.
\end{equation}
\end{theorem}
\begin{proof}[Proof]
Applying the contractivity property $||| \Phi_\Omega(g) - \Phi_\Omega(h)||| \leq \| g-h \|_2$ with $g=F_\tau f$ and $h=f$, and using \eqref{eq:main} yields \eqref{mainmainmain} upon invoking the same arguments as in \cite[Eq. 58]{Wiatowski_journal} and \cite[Lemma 2]{Wiatowski_journal} to conclude that $f\in L^2(\Rd)$ implies $F_\tau f \in L^2(\Rd)$ thanks to $\| D\tau \|_\infty \leq \frac{1}{2d}$.
\end{proof}
The strength of the deformation stability result in Theorem \ref{mainmain} derives itself from the fact that the only condition we need to impose on the underlying module-sequence $\Omega$ is weak admissibility according to \eqref{weak_admiss2}, which as argued in \cite[Sec. 3]{Wiatowski_journal}, can easily be met by normalizing the elements in $\Psi_n$, for all $n\in \mathbb{N}$, appropriately. We emphasize that this normalization does not have an impact on the constant $C_K$ in \eqref{mainmainmain}, which is shown in Appendix \ref{app:prothm1} to be independent of $\Omega$. The dependence of $C_K$ on $K$ does, however, reflect the intuition that the deformation stability bound should depend on the signal class description complexity. For band-limited signals, this dependence is exhibited by the RHS in \eqref{eq:oldstab} being linear in the bandwidth $R$. 
Finally, we note that the vertical translation invariance result \cite[Thm. 2]{Wiatowski_journal} applies to all $f\in L^2(\Rd)$, and,  thanks to $\mathcal{C}^K_{\mathrm{CART}}\subseteq L^2(\Rd)$, for all $K>0$, carries over to cartoon functions. 
\begin{remark}
We note that thanks to the decoupling technique underlying our arguments, the deformation stability bounds \eqref{eq:oldstab} and \eqref{mainmainmain} are very general in the sense of applying to every contractive (linear or non-linear) mapping $\Phi$. Specifically, the identity mapping $\Phi(f) = f$ also leads to deformation stability on the class of cartoon functions (and the class of band-limited functions). This is interesting as it was recently demonstrated that employing the identity mapping as a so-called \textit{shortcut-connection} in a subset of layers of a \textit{very deep} convolutional neural network yields state-of-the-art classification performance on the ImageNet dataset \cite{he2015deep}. Our deformation stability result is hence general in the sense of applying to a broad class of network architectures used in practice.
\end{remark}
For functions that do not exhibit discontinuities along $C^2$-hypersurfaces, but otherwise satisfy the decay condition \eqref{eq:decay}, we can improve the decay rate of the deformation error from $\alpha=\frac{1}{2}$ to $\alpha=1$. 
\begin{corollary}\label{cor:main}
 Let $\Omega=\big( (\Psi_n,M_n,R_n)\big)_{n \in \mathbb{N}}$ be a module-sequence satisfying the weak admissibility condition \eqref{weak_admiss2}. For every size $K>0$, the feature extractor $\Phi_\Omega$ is deformation-stable on the space $H_K:=\{ f \in L^2(\Rd)\,\cap \,C^1(\Rd,\C) \ | \ |\nabla f(x)|\le K \langle x\rangle^{-d} \}$ w.r.t. deformations $(F_{\tau} f)(x)=f(x-\tau(x))$, i.e., for every $K>0$, there exists a constant $C_K>0$ (that does not depend on $\Omega$) such that for all $f \in H_K$, and all (possibly non-linear) $\tau \in C^1(\Rd,\Rd)$ with $\| \tau\|_\infty <\frac{1}{2}$ and $\| D \tau \|_\infty\leq\frac{1}{2d}$, it holds that
\begin{equation*}
||| \Phi_\Omega(F_{\tau} f)-\Phi_\Omega(f) |||\leq C_K \| \tau \|_\infty.
\end{equation*}
\end{corollary}
\begin{proof}
The proof follows that of Theorem \ref{mainmain} apart from employing \eqref{eq:p1p1} instead of \eqref{eq:main}.\end{proof}
\appendices
\section{Proof of Proposition \ref{prop:main}}\label{app:prothm1}
The proof of \eqref{eq:main} is based on judiciously combining deformation stability bounds for the components $f_1,f_2$ in $(f_1+\mathds{1}_Bf_2) \in \mathcal{C}^K_{\mathrm{CART}}$ and for the indicator function $\mathds{1}_B$. The first bound, stated in Lemma \ref{lem:lip} below, reads
 \begin{equation}\label{eq:main3}
  	\|f - F_\tau f\|_2\le C  D\|\tau\|_\infty,
  \end{equation}
and applies to functions $f$ satisfying the decay condition \eqref{eq:lll}, 
with the constant $C>0$ as defined in \eqref{eq:lll} and $D>0$ not depending on $f$, $\tau$ (see \eqref{eq:ll}). The bound in \eqref{eq:main3} requires  the assumption $\| \tau \|_\infty<\frac{1}{2}$. The second bound, stated in Lemma \ref{lem:ind} below,  is 
  \begin{equation}\label{eq:main2}
  	\|\mathds{1}_B - F_\tau \mathds{1}_B \|_2\le  C_{\partial B}^{1/2}   \|\tau\|_\infty^{1/2},
  \end{equation}
where the constant $C_{\partial B}>0$ is independent of $\tau$.
 We now show how \eqref{eq:main3} and \eqref{eq:main2} can be combined to establish  \eqref{eq:main}. For $f=(f_1+\mathds{1}_Bf_2)\in \mathcal{C}^K_{\mathrm{CART}}$, we have
\begin{align}
&\|f-F_\tau f \|_2\leq \| f_1 - F_\tau f_1 \|_2\nonumber \\
&+\|\mathds{1}_B(f_2-F_\tau f_2) \|_2 + \|  (\mathds{1}_B - F_\tau \mathds{1}_B)(F_\tau f_2)\|_2 \label{eq:p1}\\
&\leq  \hspace{-0.05cm}\| f_1 - F_\tau f_1 \|_2 \hspace{-0.05cm}+\|f_2-F_\tau f_2 \|_2 \hspace{-0.07cm}+ \| \mathds{1}_B - F_\tau \mathds{1}_B\|_2\| F_\tau f_2 \|_\infty,\nonumber
\end{align}
where in \eqref{eq:p1} we used $(F_\tau(\mathds{1}_Bf_2))(x)=(\mathds{1}_Bf_2)(x-\tau(x))=\mathds{1}_B(x-\tau(x))f_2((x-\tau(x)))=(F_\tau\mathds{1}_B)(x)(F_\tau f_2)(x)$. With the upper bounds \eqref{eq:main3} and \eqref{eq:main2}, invoking properties of the class of cartoon functions $\mathcal{C}^K_{\mathrm{CART}}$ (namely, (i) $f_1$,$f_2$ satisfy \eqref{eq:decay} and thus, by Lemma \ref{lem:lip}, \eqref{eq:main3} with $C=K$, and (ii) $\| F_\tau f_2 \|_\infty =\sup_{x\in \Rd}|f_2(x-\tau(x))|\leq \sup_{y\in \Rd}|f_2(y)|=\| f_2\|_\infty\leq K)$, this yields 
\begin{align*}
\|f-F_\tau f \|_2 \leq&  \, 2\, KD\, \| \tau \|_\infty +  K C_{\partial B}^{1/2}    \|\tau\|_\infty^{1/2}\label{eq:p1p}\\
\leq&\underbrace{2\max\{2KD,K C_{\partial B}^{1/2}\}}_{=:C_K} \| \tau \|_\infty^{1/2} \nonumber,
\end{align*}
which completes the proof of \eqref{eq:main}. 

It remains to show \eqref{eq:main3} and \eqref{eq:main2}.
\begin{lemma}\label{lem:lip}
Let $f\in L^2(\Rd)\cap C^1(\R^d,\C)$ be such that
\begin{equation}\label{eq:lll}
|\nabla f(x)|\le C\langle x \rangle^{-d},
\end{equation}
for some constant $C>0$, and let $\| \tau \|_\infty<\frac{1}{2}$. Then,
	\begin{equation}\label{eq:p1p1}
		\|f - F_\tau f\|_2 \le C  D \|\tau\|_\infty,
	\end{equation}
for a constant $D>0$ that does not depend on $f$, $\tau$.
\end{lemma}
\begin{proof}
We first upper-bound the integrand in 
$
\|f - F_\tau f \|_2^2=\int_{\R^d}|f(x) - f(x-\tau(x))|^2 \mathrm dx.
$
Owing to the mean value theorem \cite[Thm. 3.7.5]{Comenetz}, we have
\begin{align*}
		|f(x) - f(x-\tau(x))|&\le  \|\tau\|_\infty\sup_{y \in B_{\|\tau\|_\infty}(x)}|\nabla f(y)|\\
		&\leq\underbrace{C\|\tau\|_\infty\sup_{y \in B_{\|\tau\|_\infty}(x)}\langle y \rangle^{-d}}_{=:h(x)},
\end{align*}
where the last inequality follows by assumption. The idea is now to split the integral $\int_{\Rd}|h(x)|^2\mathrm dx$ into integrals over the sets $B_{1}(0)$ and $\R^d \backslash B_{1}(0)$. For $x\in B_{1}(0)$, the monotonicity of the function $x\mapsto \langle x \rangle^{-d}$ implies $h(x)\leq C\|\tau\|_\infty \langle 0 \rangle^{-d}=C\|\tau\|_\infty$, and for $x\in \R^d \backslash B_{1}(0)$, we have $(1-\| \tau \|_\infty)\leq (1-\frac{\|\tau\|_\infty}{|x|})$, which together with the monotonicity of $x\mapsto \langle x \rangle^{-d}$ yields $h(x)\leq C\|\tau\|_\infty \langle (1-\frac{\|\tau\|_\infty}{|x|})x \rangle^{-d}\leq C\|\tau\|_\infty \langle (1-\|\tau\|_\infty)x \rangle^{-d}$. Putting things together, we hence get
	\begin{align}
		&\|f - F_\tau f\|_2^2 \le C^2\|\tau\|_\infty^2\Big( \mbox{vol}^{d}\big(B_{1}(0)\big)\nonumber \\
		&+2^{d}\int_{\R^d}\langle u \rangle^{-2d}\mathrm du\Big)\label{eq:p5}\\
		&\le C^2\|\tau\|_\infty^2\underbrace{\Big( \mbox{vol}^{d}\big(B_{1}(0)\big) +2^{d}\| \langle \cdot \rangle^{-d}\|_2^2\Big)}_{=:D^2},\label{eq:ll}	
		\end{align}
where in \eqref{eq:p5} we used the change of variables $u=(1-\| \tau\|_\infty)x$, together with $
\frac{\mathrm du}{\mathrm dx}=(1-\| \tau \|_\infty)^d \geq 2^{-d}$,
where the last inequality follows from $\| \tau \|_\infty<\frac{1}{2}$, which is by assumption.
Since $\| \langle \cdot \rangle^{-d}\|_2<\infty$, for $d\in \mathbb{N}$ (see, e.g., \cite[Sec. 1]{Grafakos}), and, obviously, $\mbox{vol}^{d}\big(B_{1}(0)\big)<\infty$, it follows that $D^2<\infty$, which completes the proof.
\end{proof}
\vspace{-0.2cm}
We continue with a deformation stability result for indicator functions $\mathds{1}_B$.
\begin{lemma}\label{lem:ind}
Let $B\subseteq \R^d$ be a compact domain whose boundary $\partial B$ is a compact topologically embedded $C^2$-hypersurface of $\Rd$ without boundary. Then, there exists a constant $C_{\partial B}>0$ (that does not depend on $\tau$) such that for all $\tau:\Rd \to \Rd$ with $\|\tau\|_\infty \leq 1$, it holds that %
 \begin{equation*}\label{eq:cartdef}
 	\|\mathds{1}_B - F_\tau \mathds{1}_{B} \|_2
 	\le C_{\partial B}^{1/2} \|\tau\|_\infty^{1/2}.
 \end{equation*}
\end{lemma}
\begin{proof}
In order to upper-bound 
$
		\|\mathds{1}_B - F_\tau \mathds{1}_{B} \|_2^2=\int_{\R^d}|\mathds{1}_{B}(x) - \mathds{1}_{B}(x-\tau(x))|^2 \mathrm dx,
$
we first note that the integrand $h(x):=|\mathds{1}_{B}(x) - \mathds{1}_{B}(x-\tau(x))|^2$ satisfies $h(x)=1$, for $x\in S$, where
$S:=\{ x\in \R^d \, |\, x\in B \  \text{ and }\  x-\tau(x)\notin B \}  \cup \{ x\in \R^d \, |\, x\notin B \  \text{ and }\  x-\tau(x)\in B \},$
and $h(x)=0$, for $x\in \R^d\backslash S$. Moreover, owing to  $S\subseteq \big(\partial{B} + B_{\|\tau\|_\infty}(0)\big)$, where $(\partial{B} + B_{\|\tau\|_\infty}(0))$ is a tube of radius $\|\tau\|_\infty$ around the boundary $\partial{B}$ of $B$, and \cite[Lemma 2]{Wiatowski_energy},  there exists a  constant $C_{\partial B}>0$  such that $\mbox{vol}^{d}(S)\leq \mbox{vol}^{d}(\partial{B} + B_{\|\tau\|_\infty}(0)) \leq C_{\partial B}\|\tau\|_\infty,$ for all $\tau:\Rd\to\Rd$ with $\|\tau\|_\infty \leq 1$. We therefore have
$\|\mathds{1}_B - F_\tau \mathds{1}_{B} \|_2^2=\int_{\R^d}|h(x)|^2\mathrm dx =\int_{S}1\mathrm dx = \mbox{vol}^{d}(S)\leq C_{\partial B} \|\tau\|_\infty,
$ which completes\vspace{0.1cm} the proof.\end{proof}
\vspace{-0.25cm}
\bibliographystyle{IEEEtran}
\bibliography{scatbib}

\end{document}